\pdfoutput=1

\documentclass{ecai} 

\usepackage{latexsym}
\usepackage{amssymb}
\usepackage{amsmath}
\usepackage{amsthm}
\usepackage{booktabs}
\usepackage[inline]{enumitem}
\usepackage[pdftex]{graphicx}
\usepackage[pdftex]{color}
\usepackage{caption}

\newtheorem{theorem}{Theorem}

\newtheorem{corollary}[theorem]{Corollary}
\newtheorem{proposition}[theorem]{Proposition}

\newtheorem{definition}{Definition}
\newtheorem{example}{Example}

\newcommand{\BibTeX}{B\kern-.05em{\sc i\kern-.025em b}\kern-.08em\TeX}

\newcommand{\cs}{\bot}
\newcommand{\ps}{+}
\newcommand{\ms}{-}
\newcommand{\ns}{n}

\begin{document}


\begin{frontmatter}


\paperid{123} 


\title{An Argumentative Approach for Explaining Preemption in Soft-Constraint Based Norms} 




\author[A,B]{\fnms{Wachara} \snm{Fungwacharakorn}\orcid{0000-0001-9294-3118}\thanks{Corresponding Author. Email: wacharaf@nii.ac.jp.}}
\author[A]{\fnms{Kanae} \snm{Tsushima}\orcid{0000-0002-3383-3389}}
\author[C]{\fnms{Hiroshi} \snm{Hosobe}\orcid{0000-0002-7975-052X
}}
\author[A,B]{\fnms{Hideaki} \snm{Takeda}\orcid{0000-0002-2909-7163}} 
\author[A]{\fnms{Ken} \snm{Satoh}\orcid{0000-0002-9309-4602}}

\address[A]{Center for Juris-Informatics, Research Organization of Information and Systems, Tokyo, Japan}
\address[B]{National Institute of Informatics, Tokyo, Japan}
\address[C]{Faculty of Computer and Information Sciences, Hosei University, Tokyo, Japan}


\begin{abstract} Although various aspects of soft-constraint based norms have been explored, it is still challenging to understand preemption. Preemption is a situation where higher-level norms override lower-level norms when new information emerges. To address this, we propose a derivation state argumentation framework (DSA-framework). DSA-framework incorporates derivation states to explain how preemption arises based on evolving situational knowledge. Based on DSA-framework, we present an argumentative approach for explaining preemption. We formally prove that, under local optimality, DSA-framework can provide explanations why one consequence is obligatory or forbidden by soft-constraint based norms represented as logical constraint hierarchies. 


\end{abstract}

\end{frontmatter}


\section{Introduction}

In complex situations, norms can conflict, leading to challenges in normative systems. To address this, several studies have interpreted norms as \emph{soft constraints}~\cite{dennis2016formal,greene2016embedding,hosobe2024soft,satoh2021overview}. Unlike \emph{hard constraints}, which must be exactly satisfied, soft constraints are allowed to be relaxed, enabling agents to prioritize norms based on the context. This paper focuses on constraint hierarchies~\cite{borning1992constraint}, a pioneering formalism for dealing with soft constraints. Prior work has investigated various aspects of maintaining norms represented as constraint hierarchies, including debugging norms based on user expectation~\cite{fungwacharakorn2022debugging}, revising norms based on new information~\cite{fungwacharakorn2022fundamental}, and exploring connections with case-based reasoning~\cite{fungwacharakorn2023connecting}. However, a key challenge lies in understanding preemption. Preemption refers to a situation where higher-level norms override lower-level norms as more information becomes available. If preemption is not well-understood, it can undermine trust in the normative system. Imagine a situation where an agent expects a certain consequence to be obligatory, but it is ultimately forbidden due to preemption. Without explanation, the consequence becomes unexpected and this can erode trust in the system. Therefore, explaining preemption is critical for building trust in the system and allowing agents to understand the rationale of the normative system for handling norms and preferences.

To address this, we present a novel argumentative approach for explaining preemption. Argumentative approaches are widely used for explanations in various reasoning domains~\cite{cyras2016abstract,racharak2021explanation}. Based on abstract argumentation framework~\cite{dung1995acceptability}, most approaches explored their own methods to build arguments, such as arguments built from precedent cases~\cite{cyras2016abstract} or defeasible rules~\cite{prakken2015formalization}. 
This paper proposes a derivation state argumentation framework (DSA-framework), with arguments built from derivation states and situational knowledge to understand preemption. Based on DSA-framework, we can provide explanations using dispute trees \cite{dung2007computing}. We formally prove that, if one consequence is obligatory or forbidden under local optimality, we can always find an explanation why it is.

This paper is structured as follows. Section \ref{sec:preliminaries} describes norm representation and some logical structures used in this paper. Section \ref{sec:framework} proposes DSA-framework. Section \ref{sec:explanation} presents preemption explanations based on the proposed framework. Section \ref{sec:discussion} discusses limitations of the framework and suggestions for future work. Finally, Section \ref{sec:conclusion} concludes this paper.

\section{Preliminaries}
\label{sec:preliminaries}

In this paper, we consider representing norms as logical constraints. Let $\mathcal{L}$ be a classical logical language generated from a set of propositional constants in a standard way. We write $\neg$
for negation, $\rightarrow$ for implication, $\leftrightarrow$ for equivalence, $\top$ for a tautology, $\bot$ for a contradiction, and $\vdash$ for a classical deductive monotonic consequence relation. A constraint hierarchy is typically represented as $H = \langle H_1,\ldots,H_l \rangle$, where $l$ is some positive integer, and each $H_i 
\subseteq \mathcal{L}$, called a \emph{level}, is a finite subset of logical constraints. In original definitions of constraint hierarchies \cite{borning1992constraint}, there exists a level $H_0$ consisting of \emph{required} (or \emph{hard}) constraints that must be exactly satisfied.  However, in this paper, we consider the level of hard constraints as a background theory $T_0$ to simplify other definitions. Each $H_i$ consists of \emph{preferential} (or \emph{soft}) constraints that can be relaxed if necessary. A constraint hierarchy is totally ordered, which means that a preferential level $H_i$ with smaller $i$ consists of more important constraints. 

Given a constraint hierarchy $H = \langle H_1,\ldots,H_l \rangle$ and a background theory $T_0$, we also treat $H$ as the whole set of logical constraints, that is $H = \bigcup_{i \in \{1,...,l\}} H_i$. With a general assumption that $T_0$ is consistent (i.e. $T_0 \not\vdash \bot$), we say $H$ is consistent if and only if $T_0 \cup H \not\vdash \bot$. For example, given that $T_0$ is empty, $\langle\{p\},\{q\}\rangle$ is consistent but $\langle\{p\},\{\neg p, q\}\rangle$ is not. For $\Phi \subseteq \mathcal{L}$, we also say $H$ is consistent with $\Phi$ if and only if $T_0 \cup H \cup \Phi \not\vdash \bot$.

Applying the concepts of sub-bases \cite{benferhat1993argumentative} and maximal consistent sets \cite{satoh1993computing} to constraint hierarchies, we say a constraint hierarchy $H' = \langle H'_1,\ldots,H'_l\rangle$ is a \emph{sub-base} of $H = \langle  H_1,\ldots,H_l \rangle$ if and only if $H'$ must have the same number of levels as $H$ and $H'_i \subseteq H_i$ for every $i \in \{1,...,l\}$. For example, $\langle\{p\},\{q\}\rangle$ is a sub-base of $\langle\{p\},\{\neg p, q\}\rangle$. Let $H$ be a constraint hierarchy, a \emph{sub-base space} of $H$ is a pair $(\Delta, \geq)$ where $\Delta$ is the set of all possible sub-bases of $H$ and $\geq$ is a partial order over $\Delta$, representing the preference of soft constraint relaxations. The strict order $>$ associated with $\geq$ is defined as $\delta > \delta'$ if and only if $\delta \geq \delta'$ and it is not the case that $\delta' \geq \delta$ (for $\delta$ and $\delta' \in \Delta$). The maximal element of $\geq$ is $H$ itself and the minimal element is the constraint hierarchy with the same number of levels as $H$ but all of them are empty. Corresponding to a local comparator in constraint hierarchies \cite{borning1992constraint}, we say $\geq$
is a local preference when $\delta \geq \delta'$ (for $\delta$ and $\delta' \in \Delta$) if and only if there exists $k \in \{1,\ldots,l\}$ such that $\delta'_k \subsetneq \delta_k$ and for every $i \in \{1,\ldots,l\}$ $i < k$ implies  $\delta'_i = \delta_i$. Given a sub-base space $(\Delta, \geq)$, we use the following notations:
\begin{itemize}
    \item $\Delta^\Phi$ (for $\Phi \subseteq \mathcal{L}$):  the set of all sub-bases in $\Delta$ consistent with $\Phi$,
    \item $max(D)$ (for $D \subseteq \Delta$): the set of all $\geq$-maximal elements of $D$.  
\end{itemize}

In our setting, we represent a situation as a consistent set of formulas $\Pi \subseteq \mathcal{L}$, and we represent a consequence as a formula $\psi \in \mathcal{L}$ such that $T_0 \cup \Pi \not \vdash \psi$ and $T_0 \cup \Pi \not \vdash \neg \psi$. Now, we define the concept of obligation, adapted from \cite{kowalski2018obligation}, as follows.

\begin{definition}[obligation]
\label{def:obligation}
Let $T_0$ be a background theory, $H$ be a constraint hierarchy corresponding with sub-base space $(\Delta, \geq)$. We say a consequence $\psi$ is \emph{obligatory} (resp. \emph{forbidden}) by $H$ with a situation $\Pi$ if and only if, for every $\delta \in max(\Delta^\Pi)$ $T_0 \cup \delta \cup \Pi \vdash \psi$ (resp. $\neg \psi$).    
\end{definition}

\begin{example}[overtaking]
\label{ex:overtaking}
Considering the following norms regarding overtaking, prioritized from less important to more important.
\begin{enumerate}
    \item Generally, drivers should not overtake the other car. 
    \item If the other car appears obstructed, drivers should  overtake the other car. 
    \item If the other car is in a danger zone, drivers should not overtake the other car.
\end{enumerate}
\end{example}

Omitting the background theory in this example, the norms can be represented as the constraint hierarchy $H = \langle\{p \rightarrow \neg r\}, \{q \rightarrow r\}, \{\neg r\}\rangle$ where $p$ represents "the other car is in a danger zone", $q$ represents "the other car appears obstructed", and $r$ represents "drivers should overtake the other car". The constraints are placed in a different order as constraint hierarchies prioritize constraints from left to right. There are eight sub-bases of $H$, ranked by the local preference as follows.

\begin{enumerate}
    \item $\delta_0 = \langle\{p \rightarrow \neg r\}, \{q \rightarrow r\}, \{\neg r\}\rangle = H$
    \item $\delta_1 = \langle\{p \rightarrow \neg r\}, \{q \rightarrow r\}, \{\}\rangle$
    \item $\delta_2 = \langle\{p \rightarrow \neg r\}, \{\}, \{\neg r\}\rangle$
    \item $\delta_3 = \langle\{p \rightarrow \neg r\}, \{\}, \{\}\rangle$
    \item $\delta_4 = \langle\{\}, \{q \rightarrow r\}, \{\neg r\}\rangle$
    \item $\delta_5 = \langle\{\}, \{q \rightarrow r\}, \{\}\rangle$
    \item $\delta_6 = \langle\{\}, \{\}, \{\neg r\}\rangle$
    \item $\delta_7 = \langle\{\}, \{\}, \{\}\rangle$
\end{enumerate}

Suppose the situation is that another car appears obstructed and it is in a danger zone ($\Pi = \{p,q\}$). We have that $\Delta^\Pi = \{\delta_2,\ldots,\delta_7\}$ because $\delta_0$ and $\delta_1$ are not consistent with $\Pi$. We also have that $r$ ("drivers should overtake the other car") is forbidden because $\delta_2$ is the maximal element of $\Delta^\Pi$ and $\delta_2 \cup \Pi \vdash \neg r$.

\section{Proposed Framework}
\label{sec:framework}

To leverage an argumentation framework in the norm structure, we first define derivation states and derivation state spaces as follows.

\begin{definition}[derivation state]
\label{def:derivation-state}
Let $T_0$ be a background theory, $\delta \subseteq \mathcal{L}$, $\pi \subseteq \mathcal{L}$, $\psi \in \mathcal{L}$, and $\Sigma = \{\bot, \ps, \ms, \ns \}$ be the domain of derivation states. A derivation state ($\sigma$) of $\psi$ with respect to $\delta$ and $\pi$ is defined as follows.

\begin{enumerate}
    \item $\sigma = \cs$ if $\delta$ is not consistent with $\pi$.
    \item $\sigma = \ps$ if $T_0 \cup \delta \cup \pi \vdash \psi$ and $T_0 \cup \delta \cup \pi \not \vdash \neg \psi$.
    \item $\sigma = \ms$ if $T_0 \cup \delta \cup \pi \not \vdash \psi$ and $T_0 \cup \delta \cup \pi \vdash  \neg \psi$.
    \item $\sigma = \ns$ if $T_0 \cup \delta \cup \pi \not \vdash \psi$ and $T_0 \cup \delta \cup \pi \not \vdash  \neg \psi$. 
\end{enumerate}
    
\end{definition}

\begin{definition}[derivation state space]
Let $H$ be a constraint hierarchy corresponding with sub-base space $(\Delta, \geq)$, $\Pi$ be a situation, and $\psi$ be a consequence. A \emph{derivation state space} (denoted by $\Omega$) of $\psi$  with respect to $H$ and $\Pi$ is the set $\Omega = \{\langle \delta,\pi,\sigma \rangle \in \Delta \times 2^\Pi \times \Sigma~|~\sigma$ is a derivation state of $\psi$ with respect to $\delta$ and $\pi\}$.
\end{definition}

Now, we define a \emph{DS-argument} as follows.

\begin{definition}[DS-argument]
\label{def:ds-argument}
Let $H$ be a constraint hierarchy corresponding with sub-base space $(\Delta, \geq)$ and $\Omega$ be a derivation state space. A \emph{DS-argument} from $\Omega$ is an element $\langle \delta,\pi,\sigma \rangle$ of $\Omega$ that satisfies following conditions. 
\begin{enumerate}
    \item $\sigma \neq \cs$, that is $\delta$ needs to be consistent with $\pi$.
    \item There is no $\langle \delta',\pi,\sigma' \rangle \in \Omega$ ($\pi$ is fixed) such that  $\delta' > \delta$ and $\sigma' \neq \cs$.  In other words, $\delta$ is a maximal sub-base of $\Delta$ consistent with $\pi$, or formally speaking, $\delta \in max(\Delta^\pi)$.
\end{enumerate}

\noindent For a DS-argument $\langle \delta,\pi,\sigma \rangle$, we call $\delta$ a corresponding sub-base, we call $\pi$ a situational knowledge, and we call $\sigma$ a derivation state.

\end{definition}

Table \ref{tab:dev-state} shows the derivation state space of a consequence $r$ with respect to $H$ from Example \ref{ex:overtaking} and the situation $\Pi = \{p,q\}$ to find all DS-arguments. There are four DS-arguments from this setting: $\langle \delta_0,\{\},\ms \rangle$,  $\langle \delta_0,\{p\},\ms \rangle$, $\langle \delta_1,\{q\},\ps \rangle$, and $\langle \delta_2,\{p,q\},\ms \rangle$, corresponding to the derivation states denoted by asterisks ($^*$) in the table.

\begin{table}[ht]
\caption{Derivation state space in Example \ref{ex:overtaking}}
\label{tab:dev-state}
\begin{center}
\begin{tabular}{c|cccccccc}
    $\pi \backslash \delta$ &
    $\delta_0$& $\delta_1$ & $\delta_2$ & $\delta_3$ & $\delta_4$ & $\delta_5$ & $\delta_6$ & $\delta_7$ \\
    \hline 
    $\{\}$&$\ms^*$&$\ns$&$\ms$&$\ns$&$\ms$&$\ns$&$\ms$&$\ns$\\
    $\{p\}$&$\ms^*$&$\ms$&$\ms$&$\ms$&$\ms$&$\ns$&$\ms$&$\ns$\\
    $\{q\}$&$\cs$&$\ps^*$&$\ms$&$\ns$&$\cs$&$\ps$&$\ms$&$\ns$\\
    $\{p,q\}$&$\cs$&$\cs$&$\ms^*$&$\ms$&$\cs$&$\ps$&$\ms$&$\ns$\\
\end{tabular}
\end{center}
\end{table}

Inspiring from abstract argumentation for case-based reasoning (AA-CBR) \cite{cyras2016abstract}, this paper proposes a derivation state argumentation framework (DSA-framework) based on derivation states and incremental knowledge of the situation as follows.

\begin{definition}[DSA-framework]
  Let $H$ be a constraint hierarchy corresponding with sub-base space $(\Delta, \geq)$, $\Pi$ be a situation, and $\psi$ be a consequence with $\Omega$ as a derivation state space of $\psi$ with respect to $H$ and $\Pi$. A DSA-framework with respect to $H$, $\Pi$, and $\psi$ is $(AR,attacks)$ satisfying the following conditions.
  \begin{enumerate}
      \item $AR$ is the set of all DS-arguments from $\Omega$.
      \item For $\langle \delta,\pi,\sigma \rangle, \langle \delta',\pi',\sigma' \rangle \in AR$,  $\langle \delta,\pi,\sigma \rangle$ attacks $\langle \delta',\pi',\sigma' \rangle$ if and only if
      \begin{itemize}
          \item (change derivation state) $\sigma \neq \sigma'$, and 
          \item (gain more knowledge) $\pi' \subsetneq \pi$, and 
          \item (concise attack)  $\nexists\langle\delta'',\pi'',\sigma\rangle \in AR$ with $\pi' \subsetneq \pi'' \subsetneq \pi$. 
      \end{itemize}
  \end{enumerate}
\end{definition}

From Example \ref{ex:overtaking}, the DSA-framework with respect to $H$, the situation $\Pi = \{p,q\}$, and a consequence $r$ can be illustrated in Figure \ref{fig:dsa-framework}. The arrows represent attacks between arguments.

\begin{figure}[ht]
\begin{center}
\includegraphics[scale=0.5]{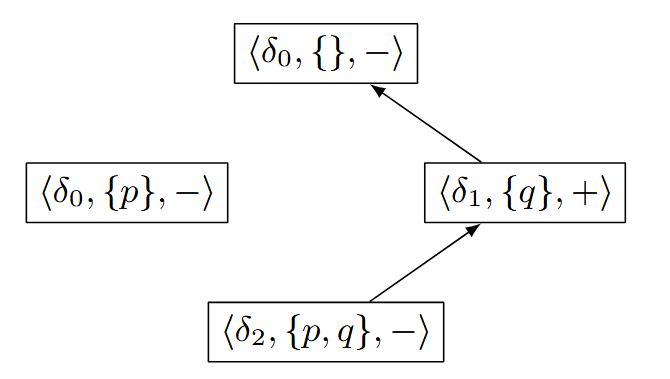}
\caption{DSA-framework from Example \ref{ex:overtaking}}
\label{fig:dsa-framework}
\end{center}
\end{figure}

Proposition \ref{prop:gen-prop} shows general properties of DSA-framework.

\begin{proposition}
\label{prop:gen-prop}
DSA-framework has the following properties.
\begin{enumerate}
    \item For every $\pi \subseteq \Pi$ (including $\{\}$ and $\Pi$), the DSA-framework has at least one argument with a situational knowledge $\pi$.
    \item DSA-framework is well-founded (i.e., acyclic).
    \item If $\psi$ is obligatory (resp. forbidden) by $H$ with a situation $\Pi$, every argument with the complete situational knowledge ($\Pi$) has a derivation state $\ps$ (resp. $\ms$).
    \item If an argument $\langle \delta, \pi, \sigma \rangle$ attacks $\langle \delta', \pi', \sigma' \rangle$, $\delta \not > \delta'$.
\end{enumerate}
\end{proposition}

\begin{proof}
    ~
    \begin{enumerate}
        \item Property 1 holds because we assume that the background theory and the situation are consistent and there exists a minimal sub-base in the sub-base space, which is the empty constraint hierarchy.
        \item Property 2 holds because we derive attacks from subset relation, which is a partial order.
        \item Property 3 follows Definition \ref{def:obligation} and Definition \ref{def:ds-argument} since for every $\delta \in max(\Delta^\Pi)$, $T_0 \cup \delta \cup \Pi \vdash \psi$ so every $\langle\delta,\Pi,\sigma\rangle \in AR$,  $\sigma = \ps$. 
        The forbidden can be proved analogously.
        \item Suppose an argument $\langle\delta, \pi, \sigma\rangle$ attacks $\langle\delta', \pi', \sigma'\rangle$ and $\delta > \delta'$. Since $\delta$ is consistent with $\pi$ and $\pi' \subsetneq \pi$, $\delta$ is consistent with $\pi'$. Together with $\delta > \delta'$, we can conclude that $\delta' \notin max(\Delta^{\pi'})$, contradicting Definition \ref{def:ds-argument}.
    \end{enumerate}
\end{proof}

Next, we consider a specific condition, called \emph{local optimality}, defined as follows.

\begin{definition}[local optimality]
\label{def:pare-opt}
Let $T_0$ be a background theory. A constraint hierarchy $H$ corresponding with sub-base space $(\Delta, \geq)$ is \emph{locally optimized} for a consequence $\psi$ with respect to a situation $\Pi$ if and only if the following conditions hold. 

\begin{enumerate}
    \item $\geq$ is a local preference.
    \item Every maximal consistent subset of $H \cup \Pi$ is decisive with respect to $\psi$, i.e. if $S \subseteq H \cup \Pi$ is consistent and no other consistent $S' \subseteq H \cup \Pi$ such that $S \subsetneq S'$ then either $T_0 \cup S \vdash \psi$ or $T_0 \cup S \vdash \neg \psi$.
\end{enumerate}

If $\psi$ is obligatory (resp. forbidden) by $H$ with $\Pi$ and 
$H$ is locally optimized for $\psi$ with respect to $\Pi$, then we say $\psi$ is locally optimally obligatory (resp. forbidden).

\end{definition}

For example, the consequence $r$ in Example \ref{ex:overtaking} is locally optimally forbidden because we use the local preference and every maximal consistent subset is decisive with respect to $r$. Under local optimality, we can simplify arguments in DSA-framework based on derivation states as follows.



\begin{proposition}
\label{prop:spec-prop}
    If a consequence $\psi$ is locally optimally obligatory by a constraint hierarchy $H$ with a situation $\Pi$, a DSA-framework with respect to $H$, $\Pi$, and $\psi$ has the following properties. 
    \begin{enumerate}
        \item All arguments with derivation states ($\ns$) do not attack any arguments.
        \item All arguments with derivation states ($\ns$) are attacked by some arguments.
        \item All arguments with derivation states ($\ms$) are attacked by some arguments.
    \end{enumerate}
\end{proposition}

\begin{proof}
    ~
    \begin{enumerate}
        \item Suppose an argument $\langle\delta^\ns,\pi^\ns,\ns\rangle$ attacks an argument $\langle\delta,\pi,\sigma\rangle$. We have that $\sigma \neq \ns$ and $\pi \subsetneq \pi^\ns$. There must be a soft constraint $c \in \delta \setminus \delta^\ns$ at level $l$ such that $T_0 \cup \{c\} \cup \pi \vdash \psi$ (if $\sigma = \ps$) or $T_0 \cup \{c\} \cup \pi \vdash \neg \psi$ (if $\sigma = \ms$). From Definition \ref{def:derivation-state}, we have that $T_0 \cup \delta^\ns \cup \pi^\ns \not \vdash \psi$ and $T_0 \cup \delta^\ns \cup \pi^\ns \not \vdash \neg \psi$.Therefore, adding $c$ into the level $l$ of $\delta^\ns$ gets $\delta'$, which is consistent with $\pi^\ns$ and $\delta' > \delta^\ns$ under the local preference, contradicting the fact that $\delta^\ns \in max(\Delta^{\pi^\ns})$ according to Definition \ref{def:ds-argument}.
        
        \item Suppose an argument $\langle\delta^\ns,\pi^\ns,\ns\rangle$ is unattacked, we have two cases:
        \begin{enumerate}[label=(\alph*)]
            \item $\forall\langle\delta^\sigma,\pi^\sigma,\sigma\rangle \in AR $ with $ \sigma \neq \ns~[~\pi^\ns \not\subset \pi^\sigma$ or $\pi^\ns = \pi^\sigma~]$: This case contradicts the fact that DSA-framework has such $\langle\delta,\Pi,\ps\rangle \in AR$ and no such $\langle\delta,\Pi,\ns\rangle \in AR$ according to Proposition \ref{prop:gen-prop}.
            
            \item $\forall\langle\delta^\sigma,\pi^\sigma,\sigma\rangle \in AR $ with $ \sigma \neq \ns~[~\exists\langle\delta',\pi',\sigma\rangle \in AR$ with $\pi^\ns \subsetneq$ $ \pi' \subsetneq \pi^\sigma$~]: This case contradicts the facts that $S = $ $\{\pi~|~\langle\delta^\sigma,\pi,\sigma\rangle \in AR$ and $ \sigma \neq \ns$ and $\pi^\ns \subsetneq \pi\}$ is not empty since $\langle\delta,\Pi,\ps\rangle \in S$ and $\pi^\ns \neq \Pi$, $S$ has a minimal element $\pi^\sigma$ with respect to the set inclusion, and no $\pi' \in S$ such that $\pi^\ns \subsetneq \pi' \subsetneq \pi^\sigma$.   
        \end{enumerate}

        \item Suppose an argument $\langle\delta^\ms,\pi^\ms,\ms\rangle$ is unattacked, we have two cases:
        \begin{enumerate}[label=(\alph*)] 

            \item $\forall\langle\delta^\ps,\pi^\ps,\ps\rangle \in AR~[~\pi^\ms \not\subset \pi^\ps$ or $\pi^\ms = \pi^\ps~]$: This case contradicts the fact that DSA-framework has such $\langle\delta,\Pi,\ps\rangle \in AR$ and no such $\langle\delta,\Pi,\ms\rangle \in AR$ according to Proposition \ref{prop:gen-prop}.
            
            \item $\forall\langle\delta^\ps,\pi^\ps,\ps\rangle \in AR~[~\exists\langle\delta',\pi',\ps\rangle \in AR$ with $\pi^\ms \subsetneq$ $ \pi' \subsetneq \pi^\ps$~]: This case contradicts the facts that $S = $ $\{\pi~|~\langle\delta^\ps,\pi,\ps\rangle \in AR$ and $\pi^\ms \subsetneq \pi\}$ is not empty since $\langle\delta,\Pi,\ps\rangle \in S$ and $\pi^\ms \neq \Pi$, $S$ has a minimal element $\pi^\ps$ with respect to the set inclusion, and no $\pi' \in S$ such that $\pi^\ms \subsetneq \pi' \subsetneq \pi^\ps$.   
        \end{enumerate}
    \end{enumerate}
    
\end{proof}

\begin{corollary}
\label{corol:3}
    If a consequence $\psi$ is locally optimally forbidden by a constraint hierarchy $H$ with a situation $\Pi$, a DSA-framework with respect to $H$, $\Pi$, and $\psi$ has the following properties.  
    \begin{enumerate}
        \item All arguments with derivation states ($\ns$) do not attack any arguments. 
        \item All arguments with derivation states ($\ns$) are attacked by some arguments
        \item All arguments with  derivation states ($\ps$) are attacked by some arguments.
    \end{enumerate}
\end{corollary}

\section{Explaining Preemption}
\label{sec:explanation}

In this section, we focus on explaining preemption with DSA-framework. Since DSA-framework is a specific type of abstract argumentation framework, we provide explanations using dispute trees in the same manner of other abstract argumentation based systems \cite{cyras2016abstract, dung2006dialectic,dung2007computing}. Referring to the original abstract argumentation framework \cite{dung1995acceptability}, we use a term \emph{AA-framework}, denoted by a pair $(\mathcal{A}, \mathcal{R})$ where $\mathcal{A}$ is a
set whose elements are called \emph{arguments} and $\mathcal{R} \subseteq \mathcal{A} \times \mathcal{A}$. For $x,y \in \mathcal{A}$, we say $x$ attacks $y$ if $\langle x,y \rangle \in \mathcal{R}$. We follow the definitions of dispute trees in AA-CBR \cite{cyras2016abstract} as follows.

\begin{definition}[dispute tree] Let $(\mathcal{A}, \mathcal{R})$ be an AA-framework. A \emph{dispute tree} for an argument $x_0 \in \mathcal{A}$, is a (possibly infinite) tree $\mathcal{T}$ with the following conditions.
\begin{enumerate}
    \item Every node of $\mathcal{T}$ is of the form $[L:x]$, with $L \in \{P, O\}$ and $x \in \mathcal{A}$ where $L$ indicates the status of proponent ($P$) or opponent ($O$).
    
    \item The root of $\mathcal{T}$ is $[P:x_0]$.
    
    \item For every proponent node $[P:y]$ in $\mathcal{T}$ and for every $x \in \mathcal{A}$ such that $x$ attacks $y$, there exists $[O:x]$ as a child of $[P:y]$.
    
    \item For every opponent node $[O:y]$ in $\mathcal{T}$, there exists at most one child of $[P:x]$ such that $x$ attacks $y$.
    
    \item there are no other nodes in $\mathcal{T}$ except those given by 1-4.
    \end{enumerate}
A dispute tree $\mathcal{T}$
is an \emph{admissible} dispute tree if and only if 
\begin{enumerate*}[label=(\alph*)]
    \item every opponent node $[O:x]$ in $\mathcal{T}$ has a child, and
    \item no $[P:x]$ and $[O:y]$ in $\mathcal{T}$ such that $x=y$.
\end{enumerate*}
A dispute tree $\mathcal{T}$ is a \emph{maximal} dispute tree if and only if for all opponent nodes $[O:x]$ which are leaves in $\mathcal{T}$ there is no argument $y \in \mathcal{A}$ such that $y$ attacks $x$.
\end{definition}

As DSA-framework is an abstract argumentation based system, similar to AA-CBR \cite{cyras2016abstract}, we adapt the definitions from AA-CBR to provide novel explanations for why a consequence is obligatory or forbidden as follows.

\begin{definition}[explanation]
    Explanations for why a consequence $\psi$ is obligatory by a constraint hierarchy $H$ with a situation $\Pi$ are:
    \begin{itemize}
        \item any admissible dispute tree for every argument $\langle\delta, \{\}, \ps\rangle$ and for every argument $\langle\delta', \pi, \ps\rangle$ that attacks $\langle\delta'', \{\}, \ns\rangle$, and
        \item any maximal dispute tree for every argument $\langle\delta, \{\}, \ms\rangle$ and for every argument $\langle\delta', \pi, \ms\rangle$ that attacks $\langle\delta'', \{\}, \ns\rangle$.
    \end{itemize}

    \noindent Explanations for why a consequence $\psi$ is forbidden by a constraint hierarchy $H$ with a situation $\Pi$ are
    \begin{itemize}
        \item any admissible dispute tree for every argument $\langle\delta, \{\}, \ms\rangle$ and for every argument $\langle\delta', \pi, \ms\rangle$ that attacks $\langle\delta'', \{\}, \ns\rangle$, and
        \item any maximal dispute tree for every argument $\langle\delta, \{\}, \ps\rangle$ and for every argument $\langle\delta', \pi, \ps\rangle$ that attacks $\langle\delta'', \{\}, \ns\rangle$.
    \end{itemize}
\end{definition}

Figure \ref{fig:explanation-1} illustrates an explanation for why $r$ ("drivers should overtake the other car") is forbidden by $H$ in Example \ref{ex:overtaking} with the situation $\Pi = \{p,q\}$. It demonstrates the preemption through the constraint hierarchy and incremental knowledge of the situation. This explanation can be interpreted into the following dialogue.

\begin{figure}[ht]
\centering
    \includegraphics[scale=0.5]{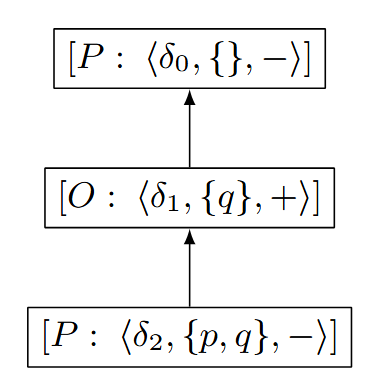}
\caption{Explanation for why $r$ is forbidden in Example \ref{ex:overtaking} with the situation $\Pi = \{p,q\}$\\~}
\label{fig:explanation-1}
\end{figure}

\begin{itemize}
    \item[] $P:$ Generally, drivers should not overtake the other car $\langle\delta_0,\{\},\ms\rangle$
    \item[] $O:$ But, the other car appears obstructed in this situation so drivers should overtake the other car $\langle\delta_1,\{q\},\ps\rangle$
    \item[] $P:$ But, the other car is in a danger zone in this situation so drivers still should not overtake the other car $\langle\delta_2,\{p,q\},\ms\rangle$
\end{itemize}

On the other hand, Figure \ref{fig:explanation-2} illustrates an explanation for why $r$ is obligatory by  $H$ in the same example but with the situation $\Pi = \{q\}$. This explanation is now a maximal dispute tree, unlike the previous explanation, which is an admissible dispute tree.

\begin{figure}[ht]
\centering
    \includegraphics[scale=0.5]{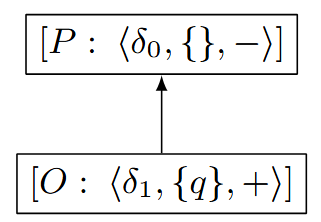}
\caption{Explanation for why $r$ is obligatory in Example \ref{ex:overtaking} with the situation $\Pi = \{q\}$\\~}
\label{fig:explanation-2}
\end{figure}

Proving existence of explanations is based on several notions from abstract argumentation framework so we recap them as follows.

\begin{definition}[from  \cite{dung1995acceptability}] 
Let $(\mathcal{A}, \mathcal{R})$ be an AA-framework, $E \subseteq \mathcal{A}$, and $x,y \in \mathcal{A}$.

\begin{enumerate}
    \item $E$ \emph{attacks} $x$ if some argument $z \in E$ attacks $x$.
    \item $E$ \emph{defends} $y$ if, for every $x \in \mathcal{A}$ that attack $y$, $E$ attacks $x$.
    \item $E$ is \emph{conflict-free} if no $x,y \in E$ such that $x$ attacks $y$.
    \item $E$ is \emph{admissible} if $E$ is conflict-free and $E$ defends every $z \in E$.
    \item $E$ is the \emph{grounded extension} of the AA-framework if it can be constructed inductively as $E = \bigcup_{i \geq 0} E_i$, where $E_0$ is the set of unattacked arguments, and $\forall i \geq 0$, $E_{i+1}$ is the set of arguments that $E_{i}$ defends.
    \item $E$ is a \emph{stable extension} of the AA-framework if it is a conflict-free set that attacks every argument that does not belong in $E$.
    \item $E$ is a \emph{preferred extension} of the AA-framework if it is a maximal admissible set with respect to the set inclusion.
    \item $E$ is a \emph{complete extension} of the AA-framework if it is an admissible set and every argument that $E$ defends, belongs to $E$.
\end{enumerate}
\end{definition}

Since DSA-framework is well-founded, the extension of the framework is unique, namely there is only one extension that is grounded, stable, preferred, and complete \cite{dung1995acceptability}.

\begin{proposition}
\label{prop:extension}
     If a consequence $\psi$ is locally optimally obligatory by a constraint hierarchy $H$ with a situation $\Pi$, an extension $E$ of DSA-framework $(AR, attacks)$ with respect to $H$, $\Pi$, and $\psi$ has the following properties. 
    \begin{enumerate}
        \item For every $\langle\delta,\{\},\ps\rangle \in AR$, $\langle\delta,\{\},\ps\rangle \in E$.
        \item For every $\langle\delta,\{\},\ms\rangle \in AR$, $\langle\delta,\{\},\ms\rangle \notin E$.
        \item For every $\langle\delta,\{\},\ns\rangle \in AR$, 
        \begin{itemize}
            \item if it is attacked by $\langle\delta^\ps,\pi^\ps,\ps\rangle \in AR$, $\langle\delta^\ps,\pi^\ps,\ps\rangle \in E$; and
            \item if it is attacked by $\langle\delta^\ms,\pi^\ms,\ms\rangle \in AR$, $\langle\delta^\ms,\pi^\ms,\ms\rangle \notin E$.
        \end{itemize}
    \end{enumerate}
\end{proposition}

\begin{proof} 
    ~
\begin{enumerate}
    \item Suppose there is $\langle\delta,\{\},\ps\rangle \in AR\setminus E$, there must be some $\langle\delta',\pi',\ms\rangle \in E$ that attacks $\langle\delta,\{\},\ps\rangle$ because $E$ is a stable extension. According to Proposition \ref{prop:spec-prop}, $\langle\delta',\pi',\ms\rangle \in E$ must be attacked by $\langle\delta'',\pi'',\ps\rangle \in AR$. We have $\langle\delta'',\pi'',\ps\rangle \notin E$ since $E$ is conflict-free: In this case, there must be $\langle\delta''',\pi''',\ms\rangle \in E$ attacks $\langle\delta'',\pi'',\ps\rangle$ and inductively we have that there must be some $\langle\delta^*,\pi^*,\ms\rangle \in E$ that is unattacked, contradicting Proposition \ref{prop:spec-prop}.
    \item Suppose there is $\langle\delta,\{\},\ms\rangle \in E$. According to Proposition \ref{prop:spec-prop}, there must be some $\langle\delta',\pi',\ps\rangle \in AR \setminus E$ that attacks $\langle\delta,\{\},\ms\rangle$. Since $E$ is admissible, there must be $\langle\delta'',\pi'',\ms\rangle \in E$ attacks $\langle\delta',\pi',\ps\rangle$ and inductively we have that there must be some $\langle\delta^*,\pi^*,\ms\rangle \in E$ that is unattacked, contradicting Proposition \ref{prop:spec-prop}.
    \item For every $\langle\delta,\{\},\ns\rangle \in AR$, there are two cases according to Proposition \ref{prop:spec-prop}:
    \begin{enumerate}[label=(\alph*)]
        \item It is attacked by $\langle\delta^\ps,\pi^\ps,\ps\rangle \in AR$. Suppose $\langle\delta^\ps,\pi^\ps,\ps\rangle \notin E$, we can prove in the same manner as 1. that there must be some $\langle\delta^*,\pi^*,\ms\rangle \in E$ that is unattacked, contradicting Proposition \ref{prop:spec-prop}.
        \item It is attacked by $\langle\delta^\ms,\pi^\ms,\ms\rangle \in AR$. Suppose  $\langle\delta^\ms,\pi^\ms,\ms\rangle \in E$, we can prove in the same manner as 2. that there must be some $\langle\delta^*,\pi^*,\ms\rangle \in E$ that is unattacked, contradicting Proposition \ref{prop:spec-prop}.
    \end{enumerate}
\end{enumerate}
\end{proof}

\begin{corollary}
     If a consequence $\psi$ is locally optimally forbidden by a constraint hierarchy $H$ with a situation $\Pi$, an extension $E$ of DSA-framework $(AR, attacks)$ with respect to $H$, $\Pi$, and $\psi$ has the following properties. 
    \begin{enumerate}
        \item For every $\langle\delta,\{\},\ms\rangle \in AR$, $\langle\delta,\{\},\ms\rangle \in E$.
        \item For every $\langle\delta,\{\},\ps\rangle \in AR$, $\langle\delta,\{\},\ps\rangle \notin E$.
        \item  For every $\langle\delta,\{\},\ns\rangle \in AR$, 
        \begin{itemize}
            \item if it is attacked by $\langle\delta^\ms,\pi^\ms,\ms\rangle \in AR$, $\langle\delta^\ms,\pi^\ms,\ms\rangle \in E$; and
            \item if it is attacked by $\langle\delta^\ps,\pi^\ps,\ps\rangle \in AR$, $\langle\delta^\ps,\pi^\ps,\ps\rangle \notin E$.
        \end{itemize}
    \end{enumerate}
\end{corollary}

\begin{proposition}
\label{prop:exp-exist}
    If a consequence $\psi$ is locally optimally obligatory by a constraint hierarchy $H$ with a situation $\Pi$, there is an explanation for why $\psi$ is obligatory by $H$ with the situation $\Pi$.
\end{proposition}

\begin{proof}
    If $\psi$ is obligatory by $H$ with a situation $\Pi$ and $\psi$ is local optimal, every argument $\langle\delta, \{\}, \ps\rangle$ and every argument $\langle\delta', \pi, \ps\rangle$ that attacks $\langle\delta'', \{\}, \ns\rangle$ are inside the extension of DSA-framework with respect to $H$, $\Pi$, and $\psi$ and every argument $\langle\delta, \{\}, \ms\rangle$ and every argument $\langle\delta', \pi, \ms\rangle$ that attacks $\langle\delta'', \{\}, \ns\rangle$ are outside the extension, according to Proposition \ref{prop:extension}. It is proved that there is an admissible dispute tree for every argument inside the extension and a maximal dispute tree for every argument outside the extension (see \cite{cyras2016explanation,dung2007computing}). 
\end{proof}

\begin{corollary}
\label{corol:exp-exist}
    If a consequence $\psi$ is locally optimally forbidden by a constraint hierarchy $H$ with a situation $\Pi$, there is an explanation for why $\psi$ is forbidden by $H$ with the situation $\Pi$.
\end{corollary}

\section{Discussion and Future Work}
\label{sec:discussion}

In section~\ref{sec:framework}, we present an algorithm to find DS-arguments within the derivation state space. However, finding DS-arguments does not require exploring the entire space. Instead, we only need to find maximal sub-bases that are consistent with the current situational knowledge. The problem of finding such sub-bases is known as Partial MAX-SAT (PMSAT) \cite{cha1997local}. PMSAT is a generalization of MAX-SAT problem \cite{krentel1986complexity} and decision versions of both problems are NP-complete \cite{fu2006solving}. Several PMSAT solvers have been developed to address this computational challenge \cite{el2013taxonomy,fu2006solving}. Following recent research~\cite{hosobe2024soft} that explored norms as general constraint hierarchies, the problem in that setting would be more challenging. This is because general constraint hierarchies consider error functions that returns progressively larger values as satisfaction decreases~\cite{borning1992constraint}. This allows degrees of satisfaction rather than true or false, making the formalization of sub-bases more difficult than ours. This highlights extending DSA-framework to handle general constraint hierarchies, along with other representations of norms, as one interesting future work.

In section~\ref{sec:explanation}, we prove that if one consequence is locally optimally obligatory or forbidden, there is always an explanation for why it is. Unfortunately, the converse is not true. That is, if there is an explanation for why one consequence is obligatory or forbidden, it does not guarantee that the consequence is actually obligatory or forbidden. This behavior can arise due to conflicts between norms. Example~\ref{ex:2} demonstrates one type of conflict where two norms within the same level have opposing enforcements on the same consequence. 

\begin{example}
\label{ex:2}
    Considering the constraint hierarchy $H = \langle\{p \rightarrow \neg r, q \rightarrow r\}\rangle$ and the situation $\Pi = \{p,q\}$.
\end{example}

    There are four sub-bases of $H$:  
    \begin{enumerate*}[label=(\alph*)]
        \item $\delta_0 = \langle\{p \rightarrow \neg r, q \rightarrow r\}\rangle = H$
        \item 
        $\delta_1 = \langle\{p \rightarrow \neg r\}\rangle$
        \item 
        $\delta_2 = \langle\{q \rightarrow r\}\rangle$
        \item $\delta_3 = \langle\{\}\rangle$
    \end{enumerate*}
    and under the local preference: $\delta_0 > \delta_1 > \delta_3$ and $\delta_0 > \delta_2 > \delta_3$. We have that $\Delta^\Pi = \{\delta_1,\delta_2,\delta_3\}$ because $\delta_0$ is not consistent with $\Pi$. We also have that $r$ is neither obligatory nor forbidden because $\delta_1,\delta_2  \in max(\Delta^\Pi)$, $\delta_2 \cup \Pi \vdash \neg r$ and $\delta_3 \cup \Pi \vdash r$. the DSA-framework with respect to $H$, the situation $\Pi = \{p,q\}$, and a consequence $r$ can be illustrated in Figure \ref{fig:dsa-ex-2}.

    \begin{figure}[ht]
    \centering
    \includegraphics[scale=0.5]{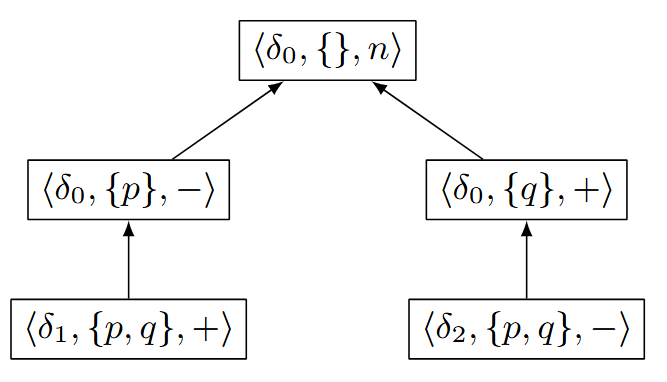}
    \caption{DSA-framework from Example \ref{ex:2}\\~\\~}
    \label{fig:dsa-ex-2}
    \end{figure}

    There is an explanation for why $r$ is obligatory (Figure \ref{fig:ex-ex-2} left) as well as  an explanation for why $r$ is forbidden (Figure \ref{fig:ex-ex-2} right). However, $r$ is neither obligatory nor forbidden as we have seen. 

    \begin{figure}[ht]
    \centering
        \includegraphics[scale=0.5]{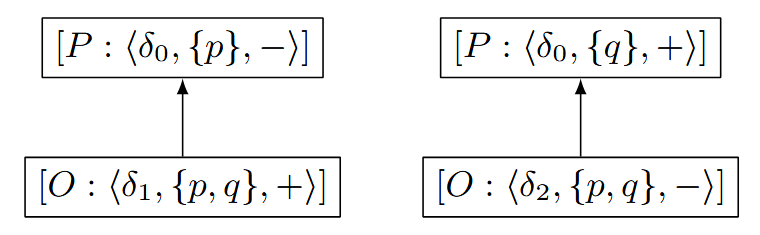}
    \caption{Explanations from Example \ref{ex:2}\\~\\~}
    \label{fig:ex-ex-2}
    \end{figure}

While this paper demonstrates the ability to explain why a consequence is obligatory, explaining why it is \emph{not} obligatory remains a challenge. This is because non-obligatory consequences can arise from either intentional permissions within the norms themselves or conflicts between norms. This highlights leveraging DSA-framework to automatically detect norm conflicts and explain non-obligatory consequences as another interesting future work.

\section{Conclusion}
\label{sec:conclusion}

This paper proposes the derivation state argumentation framework (DSA-framework) for explaining preemption in soft-constraint based norms represented as logical constraint hierarchies. The framework utilizes arguments that incorporate derivation states and the evolving knowledge of a situation.  Under the local optimality, this approach guarantees explanations for why certain consequences are obligatory or forbidden, based on the properties of arguments within the DSA-framework and its extensions. Future research directions include leveraging DSA-framework to explain non-obligatory consequences, automatically detect norm conflicts, and extend its applicability to handle general constraint hierarchies and other normative representations.


\begin{ack}
This work was supported by JSPS KAKENHI Grant Number,
JP22H00543, JST, AIP Trilateral AI Research, Grant Number, JPMJCR20G4,
and the MEXT "R\&D Hub Aimed at Ensuring Transparency and Reliability
of Generative AI Models" project.
\end{ack}



\end{document}